\documentclass[twoside]{article}
\usepackage[accepted]{aistats2e}

\usepackage{amsmath, amsthm}
\usepackage[psamsfonts]{amssymb}
\usepackage{latexsym}
\usepackage{graphics}
\usepackage{enumerate}
\usepackage{amstext}
\usepackage{url}
\usepackage{epsfig}
\usepackage{natbib}

\newcommand{\namecite}{\cite}
\newcommand{\emcite}{\cite}

\usepackage{algorithm}
\usepackage{algorithmic}

\newtheorem{inequality}{Inequality}
\newtheorem{inequalityrestate}{Inequality}

\newtheorem{thm}{Theorem}
\newtheorem{lem}[thm]{Lemma}

\newcommand{\paren}[1]{{\left({#1}\right)}}
\newcommand{\pr}[1]{{{\bf Pr}\left[{#1}\right]}}
\newcommand{\Exp}[1]{{{\bf E}\left[{#1}\right]}}
\newcommand{\Expt}[1]{{{\bf E}_t\left[{#1}\right]}}
\newcommand{\m}{{p_{\min}}}
\newcommand{\sumt}{\sum_{t=1}^T}
\newcommand{\sumi}{\sum_{i=1}^N}
\renewcommand{\eqref}[1]{{Eq.~(\ref{#1})}}
\newcommand{\bxi}{{\boldsymbol \xi}}
\newcommand{\rr}{{\boldsymbol{r}}}
\newcommand{\vh}{{\hat{v}}}
\newcommand{\ucb}{{\hat{U}}}
\newcommand{\veca}{{\bf a}}
\newcommand{\preda}{{\bxi^{\veca}}}

\newcommand{\Alg}{A}
\DeclareMathOperator*{\ExpP}{Exp4.P}

\renewcommand{\bibsection}{\subsubsection*{References}}

\begin{document}
\twocolumn[
\aistatstitle{Contextual Bandit Algorithms with Supervised Learning Guarantees} 
      
\aistatsauthor{Alina Beygelzimer \And John Langford \And  Lihong Li}
\aistatsaddress{IBM Research\\ Hawthorne, NY\\{beygel@us.ibm.com}
\And Yahoo!\ Research \\ New York, NY\\{jl@yahoo-inc.com}
\And Yahoo!\ Research\\ Santa Clara, CA\\{lihong@yahoo-inc.com}}
 \aistatsauthor{Lev Reyzin
\And Robert E. Schapire}
\aistatsaddress{ Georgia Institute of Technology\\ Atlanta, GA\\ {lreyzin@cc.gatech.edu}
\And Princeton University\\ Princeton, NJ\\ {schapire@cs.princeton.edu}}
\runningauthor{Alina Beygelzimer, John Langford, Lihong Li, Lev Reyzin, Robert E. Schapire}
]

%\maketitle

\begin{abstract}
We address the problem of competing with any large set of $N$ policies 
in the non-stochastic bandit setting, where the learner must repeatedly 
select among $K$ actions but observes only the reward of the chosen action.

We present a modification of the \texttt{Exp4} algorithm of Auer et 
al.~\cite{AuerCFS02}, called \texttt{Exp4.P}, 
which with high probability incurs regret at most 
$O(\sqrt{KT\ln N})$.
Such a bound does not hold for \texttt{Exp4} 
due to the large variance of the importance-weighted estimates 
used in the algorithm.
The new
algorithm is tested empirically in a large-scale, real-world dataset.
For the stochastic version of the problem, we can use \texttt{Exp4.P}
as a subroutine 
to compete with a possibly infinite set of policies of VC-dimension $d$
while incurring regret at most $ O(\sqrt{Td\ln T})$ with high probability.

These guarantees improve on those of all previous algorithms, whether
in a stochastic or adversarial environment, and bring us closer to 
providing guarantees for this setting that are 
comparable to those in standard supervised learning.
\end{abstract}

\section{INTRODUCTION}
A learning algorithm is often faced with the problem of acting given 
feedback only about the actions that it has taken in the past, requiring
the algorithm to explore.
A canonical example is the problem of personalized content recommendation
on web portals,
where the goal is to learn which items are of greatest interest to a user,
given such observable context as the user's search queries or geolocation.

Formally, we consider an online bandit setting where at every step,
the learner observes some contextual information and must choose one 
of $K$ actions, each with a potentially different reward on every round.
After the decision is made, the reward of the chosen action is revealed.
The learner has access to a class of $N$ policies, each of which also maps context 
to actions; the learner's performance is measured in terms of its 
\emph{regret} to this class, defined as the difference between 
the cumulative reward of the best policy in the class and the learner's reward.

This setting goes under different names,
including the ``partial-label problem"~\cite{KakadeST08},
the ``associative bandit problem"~\cite{SMLH06},
the ``contextual bandit problem"~\cite{LangfordZ07} (which is the name 
we use here),
the ``$k$-armed (or multi-armed) bandit
problem with expert advice"~\cite{AuerCFS02}, and ``associative
reinforcement learning''~\cite{Kaelbling94AssociativeFunctions}.
Policies are sometimes referred to as hypotheses or experts, and actions
are referred to as arms.

If the total number of steps $T$ (usually much larger than $K$) is known
in advance, and the contexts and rewards are sampled independently 
from a fixed but unknown joint distribution, 
a simple solution is to first choose actions
uniformly at random for $O(T^{2/3})$ rounds, and from that
point on use the policy that performed best on these rounds.
This approach, a variant of $\epsilon$-greedy (see \emcite{SuttonB98}), 
sometimes called $\epsilon$-first, can be shown to have a regret 
bound of $\textstyle O\left( T^{2/3} (K \ln N)^{1/3} \right)$
with high probability~\cite{LangfordZ07}. 
In the full-label setting, where the entire reward vector is 
revealed to the learner at the end of each step,
the standard machinery of supervised learning 
gives a regret bound 
of $\textstyle O(\sqrt{T \ln N})$
with high probability, using the algorithm that predicts according to the 
policy with the currently lowest empirical error rate.

This paper presents the first algorithm, \texttt{Exp4.P}, 
that {with high probability} achieves $O(\sqrt{T K \ln N})$ regret in 
the {adversarial} contextual bandit setting.  This improves on the 
$O(T^{2/3} (K \ln N)^{1/3})$ high probability bound in the {stochastic} 
setting.
Previously, this result was known to hold {\em in expectation\/} for the
algorithm \texttt{Exp4}~\cite{AuerCFS02}, but a {high probability\/} 
statement did not hold for the same algorithm, as per-round regrets 
on the order of $O(T^{-1/4})$ were possible~\cite{AuerCFS02}. 
Succeeding with high probability
is important because reliably useful methods are preferred in practice.

The \texttt{Exp4.P} analysis addresses competing with a finite (but
possibly exponential in $T$) set of policies. 
In the stochastic case, $\epsilon$-greedy or epoch-greedy style 
algorithms~\cite{LangfordZ07} can compete with an infinite set 
of policies with a finite VC-dimension,
but the worst-case regret grows as $O(T^{2/3})$
rather than $O(T^{1/2})$.  
We show how to use \texttt{Exp4.P} in a black-box fashion 
to guarantee a high probability regret bound 
of $O(\sqrt{Td\ln T})$ in this
case, where $d$ is the VC-dimension.
There are simple examples showing that it is impossible to compete
with a VC-set with an online adaptive adversary, so some stochastic assumption
seems necessary here.

%The technique used is to first sample for a small number of rounds,
%construct a very crude cover, then employ \texttt{Exp4.P} in a
%black-box fashion.  This approach can be easily adapted for use with
%other online exponential-weight algorithms.

This paper advances a basic argument, namely, that
such exploration problems are solvable in almost exactly the
same sense as supervised learning problems, with suitable modifications to
existing learning algorithms.  In particular, 
we show that learning to compete with any set of
strategies in the contextual bandit setting requires only a
factor of $K$ more experience than for supervised learning (to achieve
the same level of accuracy with the same confidence).

\texttt{Exp4.P} does retain one limitation of its predecessors---it
requires keeping 
explicit weights over the experts, so in the case when $N$ is too large, the algorithm
becomes inefficient.  %Abernethy and Rakhlin~\cite{AbernethyR09} 
%ask whether an efficient $\sqrt{T}$-regret 
%algorithm exists for linear multiclass predictors. This question also remains open
%for the arbitrary experts setting if we assume the learner has access to an ERM-type oracle.
On the other hand, \texttt{Exp4.P} provides a practical framework for incorporating
more expressive expert classes, and it is efficient when $N$ is polynomial
in $K$ and $T$.
It may also be possible to run \texttt{Exp4.P} efficiently in certain
cases when working with a family of experts that is exponentially
large, but well structured, as in the case of experts corresponding to
all prunings of a decision tree~\cite{HelmboldSc97}.
A concrete example of this approach is given in Section~\ref{s:experiments}, 
where an efficient implementation of \texttt{Exp4.P} is applied
to a large-scale real-world problem.

\smallskip\noindent
{\bf Related work:}\quad
The non-contextual $K$-armed bandit problem was introduced by
Robbins~\cite{Robbins52Some}, and analyzed by
Lai and Robbins~\cite{LaiR85} in the i.i.d.\ case for fixed reward
distributions.

An adversarial version of the bandit problem was introduced by
Auer~et~al.~\cite{AuerCFS02}.
They gave an exponential-weight algorithm called \texttt{Exp3} with
expected cumulative regret of $\tilde{O}(\sqrt{KT})$ and also \texttt{Exp3.P}
with a similar bound that holds with high probability.
They also showed that these are essentially optimal by proving a
matching lower bound, which holds even in the i.i.d. case.
They were also the first to consider the
$K$-armed bandit problem with expert advice, introducing the
\texttt{Exp4} algorithm as discussed earlier.
Later, McMahan and Streeter~\cite{McMahanS09} 
designed a cleaner algorithm that improves on their bounds 
when many irrelevant actions (that no expert recommends) exist.
Further background on online bandit problems appears in
\cite{Cesa-BianchiL06}.

%For high probability regret bounds in the i.i.d.\ setting, Auer~et~al.~\cite{AuerCF02}
%introduced the use of confidence bounds.
%Langford and
%Zhang~\cite{LangfordZ07} presented an algorithm with
%$O(T^{2/3})$ regret. %, assuming availability of a certain kind of oracle.
%Also, Auer~\cite{Auer02} gave a high-probability algorithm for shifting
%experts.

\texttt{Exp4.P} is based on a careful composition of the \texttt{Exp4}
and \texttt{Exp3.P} algorithms. %, which satisfies the theorem discussed above.  
We distill out the exact exponential moment method bound used
in these results, proving an inequality for martingales
(Theorem~\ref{thm:0}) to derive a sharper bound more directly.  
Our bound is a Freedman-style inequality for martingales~\cite{Freedman75}, and
a similar approach was taken in Lemma 2 of
Bartlett~et~al.~\cite{BDHKRT08}. Our
bound, however, is more elemental than
Bartlett~et~al.'s since our Theorem can be used to
prove (and even tighten) their Lemma, but not vice versa.

With respect to competing with a VC-set, 
a claim similar to our Theorem~\ref{thm:VC} (Section~\ref{sec:VCVE}) 
appears in a work of Lazaric and Munos~\cite{LazaricM09}.
Although they incorrectly claimed that \texttt{Exp4} 
can be analyzed to give a regret bound of $\tilde{O}(KT\ln{N})$ 
with high probability, one can use \texttt{Exp4.P} in their proof instead.
Besides being correct, our analysis is tighter, which is important 
in many situations where such a risk-sensitive algorithm might be applied.

Related to the bounded VC-dimension setting, Kakade and Kalai~\cite{KakadeK05} 
give a $O(T^{3/4})$ regret guarantee for the transductive online setting, 
where the learner can observe the rewards
of all actions, not only those it has taken.  
In \cite{BenDavidPS09}, Ben-David~et~al.\ consider
agnostic online learning for bounded Littlestone-dimension.  
However, as VC-dimension
does not bound Littlestone dimension, our work provides much 
tighter bounds in many cases.

\smallskip\noindent
{\bf Possible Approaches for a High Probability Algorithm}\quad

To develop a better intuition about the problem, we describe
several naive strategies and illustrate why they fail. 
These strategies fail even if the rewards of each arm are drawn independently 
from a fixed unknown distribution, and thus certainly 
fail in the adversarial setting.

{\bf Strategy 1}: Use confidence bounds to maintain a set of plausible 
experts, and randomize uniformly over the actions predicted 
by at least one expert in this set.  
To see how this strategy fails, consider two arms, 1 and 0, 
with respective deterministic rewards 1 and 0.  The expert set contains
$N$ experts.  At every round, one of them is chosen uniformly at random 
to predict arm 0, and the remaining $N-1$ predict arm 1.
%There is an additional expert which always chooses arm 1.
All of the experts have small regret with high probability.
The strategy will randomize uniformly over both arms on every round, 
incurring expected regret of nearly $T/2$.

{\bf Strategy 2}: Use confidence bounds to 
maintain a set of plausible experts, and follow the prediction of an
expert chosen uniformly at random from this set.
To see how this strategy fails, 
let the set consist of $N > 2T$ experts predicting 
in some set of arms, 
all with reward 0 at every round,  and let there be a good expert 
choosing another arm, which always has reward 1.  The probability
we never choose the good arm is $(1-1/N)^T$.  We have
$-T\log(1-\frac{1}{N}) < T\frac{\frac{1}{N}}{1-\frac{1}{N}} \le \frac{2T}{N} < 1$, 
using the elementary inequality $-\log(1-x) < x/(1-x)$ for $x\in (0,1]$.  Thus $(1-1/N)^T > \frac{1}{2}$,
and the strategy incurs regret of $T$ with probability greater than $1/2$ 
(as it only observes 0 rewards and is unable to eliminate any of 
the bad experts).

%\smallskip\noindent
%Before getting to the technical details, we give a basic overview 
%of the techniques used 
%in \texttt{Exp4.P}.
%Like its predecessors, \texttt{Exp4.P} keeps 
%a {weight} for each expert.  The weight is 
%exponential in the expert's estimated performance and a bound on the variance 
%of its performance.  
%This technique,
%used both in \texttt{Exp3.P} \cite{AuerCFS02} and \texttt{UCB}~\cite{AuerCF02},
%encourages exploring experts whose performance we are uncertain about.
%This encourages exploiting experts whose recommendations we believe to be good.

%The second key technique for encouraging exploration is to ensure
%that every action is taken with some minimum probability $\m = 
%\sqrt{(\ln N)/{TK}}$.
%A different method of enforcing this condition directly exists~\emcite{McMahanS09}
%that can improve the regret in some cases by a multiplicative constant, which is important for empirical applications. 
%Even if each round of exploration yields regret of $1$ for the \texttt{Exp4.P},
%exploration is done sufficiently infrequently not to hurt its worst-case performance.

%Finally, we use \emph{importance weighting}---a method for
%making unbiased estimates of each expert's performance when the
%probability with which we follow each expert is different.  The idea
%behind importance weighting is to scale the observed rewards for an
%expert inversely with the \emph{a priori} probability with taking its
%chosen actions.  This method is described in
%detail in \emcite{SuttonB98}.

\section{PROBLEM SETTING AND NOTATION}
%Let $T$ be the number of rounds, $K$ the number of possible arms (which
%we also call actions in this paper),
%and $N$ the number of policies in $\Pi$, which we call \textbf{experts}.    For any 
%policy $\pi$, the associated expert predicts according to $\pi(x)$ where $x$ is the 
%context available for each round.  Since the context is only used in this fashion,
%we only refer to the predictions of experts, except in the section discussing competition with a VC-set.

Let $\rr(t)\in [0,1]^K$ be the vector of
rewards, where $r_j(t)$ is the reward of arm $j$ on round
$t$.  Let $\bxi^i(t)$ be the $K$-dimensional advice vector of expert
$i$ on round $t$. This vector represents a probability distribution
over the arms, in which each entry $\xi^i_j(t)$ is the (expert's
recommendation for the) probability of choosing arm $j$.
For readability, we always use $i\in\{1,\ldots,N\}$ to index %iterate over 
experts and $j\in\{1,\ldots,K\}$ to index %iterate over
arms.

For each policy $\pi$, the associated expert predicts
according to $\pi(x_t)$, where $x_t$ is the context available in round $t$.
As the context is only used in this fashion here, we
talk about expert predictions as described above.
For a deterministic $\pi$, the corresponding prediction vector has a $1$ in
component $\pi(x_t)$ and $0$ in the remaining components.

On each round $t$, the world commits to $\rr(t) \in [0,1]^K$.
Then the $N$ experts make their recommendations
$\bxi^1(t), \ldots, \bxi^N(t)$, and
the learning algorithm $\Alg$ (seeing the recommendations but not the rewards)
chooses action $j_t\in \{1,\ldots,K\}$.
Finally, the world reveals reward $r_{j_t}(t)$ to the learner, and this game
proceeds to the next round.

We define the \emph{return} (cumulative reward) of $\Alg$ as
$
G_{\Alg} \doteq \sum_{t=1}^{T}{r_{j_t}(t)}.
$
Letting
$y_{i}(t)=\bxi^{i}(t)\cdot \rr(t)$,
we also define the expected return of expert $i$,
\[
G_{i}\doteq\sum_{t=1}^{T}y_{i}(t),
\]
and $G_{\max}=\max_i{G_i}$.
The expected \emph{regret} of algorithm $\Alg$ is defined as
\[
 G_{\max} -  \mathbf{E}[G_{\Alg}].
\]
We can also think about bounds on the regret which hold with arbitrarily high probability.
In that case, we can say that the regret is bounded by $\epsilon$ with probability $1- \delta$, if
we have
\[
\textbf{Pr}[G_{\max} -  G_{\Alg} > \epsilon] \le \delta.
\]
In the definitions of expected regret and the high probability bound, the probabilities and
expectations are taken w.r.t.\ both the randomness in the rewards $\rr(t)$ and the
algorithm's random choices.

%While a common assumption in the literature on this problem is that
%the reward vectors $\hat{\rr}(t)$ are drawn i.i.d.\ from some
%distribution, we note that for our results we make no such assumption
%about how the world chooses the rewards.  Hence, our algorithms and
%analysis hold against any adversary unaware of the choices our
%algorithm makes.

\section{A GENERAL RESULT FOR MARTINGALES}

Before proving our main result (Theorem~\ref{thm:main}), we prove a general result for martingales in which 
the variance is treated as a random variable.  It is used in the proof of Lemma
\ref{lem:1} and may also be of independent interest.  The technique is
the standard one used to prove Bernstein's inequality for martingales
\cite{Freedman75}.  The useful difference here is that we prove the
bound for any fixed \emph{estimate} of the variance rather than any
\emph{bound} on the variance.

Let $X_{1},\ldots,X_{T}$ be a sequence of real-valued random variables.
Let $\Expt{Y}=\Exp{Y | X_1,\ldots,X_{t-1}}$.

\begin{thm} \label{thm:0}
Assume, for all $t$, that $X_t\leq R$ and that $\Expt{X_t}=0$.
Define the random variables
\[
  S \doteq \sumt X_t, \qquad V \doteq \sumt \Expt{X_t^2}.
\]
Then for any $\delta>0$, with probability at least $1-\delta$, we have 
the following guarantee:

For any $V'\in \left[ \frac{R^2 \ln (1/\delta)}{e-2},\infty \right)$, 
\[
   S \leq \sqrt{(e-2)\ln(1/\delta)} \left( \frac{V}{\sqrt{V'}} + \sqrt{V'} \right) 
\]
and for $V' \in \left[0,\frac{R^2 \ln (1/\delta)}{e-2} \right]$,
\[
   S \leq R \ln (1/\delta) + (e-2) \frac{V}{R}.
\]
\end{thm}

Note that a simple corollary of this theorem is the more typical
Freedman-style inequality, which depends on an \emph{a priori} upper bound, 
which can be substituted for $V'$ and $V$.

\begin{proof}
For a fixed $\lambda \in [0,1/R]$, 
conditioning on $X_1,\ldots,X_{t-1}$ and computing expectations gives
\begin{eqnarray}
   \Expt{e^{\lambda X_t}}
 &\leq&
   \Expt{1 + \lambda X_t + (e-2) \lambda^2 X_t^2}
\label{eq:1} \\
 &=&
   1 + (e-2)\lambda^2 \Expt{X_t^2}
\label{eq:2} \\
 &\leq&
   \exp\paren{(e-2)\lambda^2 \Expt{X_t^2}}.
\label{eq:3}
\end{eqnarray}
\eqref{eq:1} uses the fact that
$e^{z} \leq 1 + z + (e-2)z^2$ for $z\leq 1$.
\eqref{eq:2} uses $\Expt{X_t}=0$.
\eqref{eq:3} uses $1+z\leq e^z$ for all $z$.

Let us define random variables $Z_0=1$ and, for $t\geq 1$,
\[
   Z_t = Z_{t-1} \cdot \exp\paren{\lambda X_t - (e-2)\lambda^2 \Expt{X_t^2}}.
\]
Then,
\begin{eqnarray*}
\Expt{Z_t}
&=&
Z_{t-1} \cdot \exp\paren{- (e-2)\lambda^2 \Expt{X_t^2}}
% \\
% & &
        \cdot \Expt{e^{\lambda X_t}}
\\
&\leq&
Z_{t-1} \cdot \exp\paren{- (e-2)\lambda^2 \Expt{X_t^2}}
 \\
 & &
        \cdot \exp\paren{(e-2)\lambda^2 \Expt{X_t^2}}
=
Z_{t-1}.
\end{eqnarray*}
Therefore, taking expectation over all of the variables
$X_1,\ldots,X_T$ gives
\[
   \Exp{Z_T} \leq \Exp{Z_{T-1}} \leq \cdots \leq \Exp{Z_0} = 1.
\]

By Markov's inequality,
$ \pr{Z_T \geq 1/\delta} \leq \delta$.
Since
\[
   Z_T = \exp\paren{\lambda S - (e-2)\lambda^2 V},
\]
we can substitute $\lambda = \min \left\{\frac{1}{R}, \sqrt{\frac{\ln (1/\delta)}{(e-2)V'}}\right\}$ and apply algebra to prove the theorem.
\end{proof}

\section{A HIGH PROBABILITY ALGORITHM}\label{ss:main}
\begin{algorithm}[t]
\begin{raggedright}
\caption{Exp4.P}\label{a:exp4.p}
\textbf{parameters:} $\delta>0,$ $\m \in [0,1/K]$ $\left(\textrm{we set }\m = \sqrt{\frac{\ln N}{KT}}\right)$
\par\end{raggedright}

\begin{raggedright}
\textbf{initialization: } Set $w_{i}(1)=1$ for $i=1,\ldots,N$.

\medskip
\par\end{raggedright}

\begin{raggedright}
\textbf{for each} $t=1,2,\ldots$ 
\par\end{raggedright}
\begin{enumerate}
\item get advice vectors $\bxi^{1}(t),\ldots,\bxi^{N}(t)$
\item set $W_{t}=\sum_{i=1}^{N}w_{i}(t)$ and for $j=1,\ldots,K$ set\[
p_{j}(t)=\left(1-K \m \right)\sum_{i=1}^{N}\frac{w_{i}(t)\xi_{j}^{i}(t)}{W_{t}}+\m \]

\item draw action $j_t$ randomly according to the probabilities $p_{1}(t),\ldots,p_{K}(t)$.
\item receive reward $r_{j_{t}}(t)\in[0,1].$
\item for $j=1,\ldots,K$ set\[
\hat{r}_{j}(t)=\begin{cases}
\begin{array}{c}
r_{j}(t)/p_{j}(t)\\
0\end{array} & \begin{array}{c}
\textrm{if }j=j_{t}\\
\textrm{otherwise}\end{array}\end{cases}\]

\item for $i=1,\ldots,N$ set
\begin{eqnarray*}
\hat{y}_{i}(t) & = & \bxi^{i}(t)\cdot\hat{\rr}(t)\\
\vh_{i}(t) & = & \sum_{j}\xi_{j}^{i}(t)/p_{j}(t)\\
w_{i}(t+1) & = & w_{i}(t)e^{\left(\frac{\m}{2}\left(\hat{y}_i(t)+
                     \vh_i(t) \sqrt{\frac{\ln(N/\delta)}{KT}}
                              \right)\right)}
\end{eqnarray*}

\end{enumerate}

\end{algorithm}

The \texttt{Exp4.P} algorithm is given in Algorithm~\ref{a:exp4.p}.
It comes with the following guarantee.
% main theorem
\begin{thm}\label{thm:main}
Assume that $\ln(N/\delta) \leq KT$, and that the set of experts includes
one which, on each round, selects an action uniformly at random.
Then, with probability at least $1-\delta$,
\[G_{\ExpP} \geq G_{\max} - 6\sqrt{KT \ln(N/\delta)}.\]
\end{thm}

The proof of this theorem relies on two lemmas.
The first lemma gives an upper confidence bound on the
expected reward of an expert given the estimated reward of that expert.

The estimated reward of an expert is defined as \[
\hat{G}_{i}\doteq\sum_{t=1}^{T}\hat{y}_{i}(t).\]
We also define
\[
\hat{\sigma}_{i}\doteq\sqrt{KT}+\frac{1}{\sqrt{KT}}\sum_{t=1}^{T} \vh_i(t).
\]

\begin{lem} \label{lem:1}
Under the conditions of Theorem~\ref{thm:main},
\[
\mathbf{Pr}\left[ \exists i:
        G_i \geq \hat{G}_{i}+ \sqrt{\ln(N/\delta)}\hat{\sigma}_{i}\right] \leq\delta.\]
\end{lem}

\begin{proof}
Fix $i$.
Recalling that $y_i(t)=\bxi^{i}(t)\cdot \rr(t)$ and the definition of
$\hat{y}_i(t)$ in Algorithm~\ref{a:exp4.p},
let us further define the random variables
$X_t = y_i(t) - \hat{y}_i(t)$ to which we will apply Theorem~\ref{thm:0}.
Then $\Expt{\hat{y}_i(t)} = y_i(t)$ so that
$\Expt{X_t}=0$ and $X_t\leq 1$.
Further, we can compute
\begin{eqnarray*}
\Expt{X_t^2} &=& \Expt{(y_{i}(t)-\hat{y}_{i}(t))^{2}} \\
& = & \Expt{\hat{y}_i(t)^2} -  y_i(t)^2 \leq \mathbf{E}_{t}[\hat{y}_{i}(t)^{2}] \\
& = & \mathbf{E}_{t}\left[\left(\bxi^{i}(t)\cdot\hat{\rr}(t)\right)^{2}\right] \\
& = & \sum_{j}p_{j}(t)\left(\xi_{j}^{i}(t)\cdot\frac{r_{j}(t)}{p_{j}(t)}\right)^{2} \\
& \le & \sum_{j}\frac{{\xi_{j}^{i}(t)}}{p_{j}(t)} \\
& = & \vh_{i}(t).
\end{eqnarray*}

Note that
\[ G_i - \hat{G}_i = \sum_{t=1}^T{X_{t}}. \]
Using $\delta/N$ instead of $\delta$, and setting $V'=KT$
in Theorem~\ref{thm:0} gives us
\begin{eqnarray*}
&\pr{G_i - \hat{G}_i \geq \sqrt{(e-2)\ln\left(\frac{N}{\delta}\right)}
\left(\frac{\sum_{t=1}^T{\vh_i(t)}}{\sqrt{KT}}+\sqrt{KT} \right)}& \\
&\le& \\ 
&\delta/N&.
\end{eqnarray*}
Noting that $e-2< 1$,
%we get:
%\begin{equation*}
%\pr{G_i - \hat{G}_i \geq \sqrt{\ln\left(\frac{N}{\delta}\right)}
%\left(\frac{\sum_{t=1}^T{\vh_i(t)}}{\sqrt{KT}}+\sqrt{KT} \right)} \le \frac{\delta}{N}
%\end{equation*}
and applying a union bound over the $N$ experts
gives the statement of the lemma.
\end{proof}

To state the next lemma,  define
$$
   \ucb = \max_i \paren{\hat{G}_i + \hat{\sigma}_i
                             \cdot\sqrt{\ln(N/\delta)} }.
$$

\begin{lem} \label{lem:2}
Under the conditions of Theorem~\ref{thm:main},
\begin{eqnarray*}
G_{\ExpP} &\ge&
   \left(1-2\sqrt{\frac{K \ln N}{T}}\right)\ucb
   - 2\sqrt{KT \ln(N/\delta)} \\
   & &
   - \sqrt{KT \ln N}
   - \ln(N/\delta).
\end{eqnarray*}
\end{lem}

We can now prove Theorem~\ref{thm:main}.% main theorem
\begin{proof} 
Taking the statement of Lemma~\ref{lem:2} and applying the result of Lemma~\ref{lem:1},
and we get, with probability at least $1-\delta$,
\begin{eqnarray}
G_{\textrm{Exp4.P}}
\label{eq:gmaxlesst}&\geq&
    G_{\max} - 2\sqrt{\frac{K \ln N}{T}} T
       - \ln(N/\delta)\\
       &&
   - \sqrt{KT \ln N}
      - 2\sqrt{KT \ln(N/\delta)}\nonumber\\
%&=&
%    G_{\max} - 3 \sqrt{KT \ln N}
%       - \ln(N/\delta)\nonumber \\
%&&   - 2 \sqrt{KT \ln(N/\delta)}\nonumber\\
&\geq&
    G_{\max} - 6 \sqrt{KT \ln(N/\delta)}\nonumber,    
\end{eqnarray}
with Eq.\ (\ref{eq:gmaxlesst}) using $G_{\max}\leq T$.
\end{proof}

\section{COMPETING WITH SETS OF FINITE VC DIMENSION}\label{sec:VCVE}

A standard VC-argument in the online setting can be used 
to apply \texttt{Exp4.P} to compete
with an infinite set of policies $\Pi$ with a finite VC dimension $d$,
when the data is drawn independently from a fixed, unknown
distribution.  For simplicity, this section assumes that there are
only two actions ($K=2$), as that is standard for the definition of VC-dimension.
%A staightforward extension to $K>2$ follows using 
%a generalization of Sauer's lemma to multivalued functions for different generalizations
%of the VC-dimension~\cite{HL95}.

The algorithm \texttt{VE} chooses an action uniformly at random 
for the first $\tau = \sqrt{T (2d \ln \frac{eT}{d} +\ln \frac{2}{\delta})}$ 
rounds.  This step partitions $\Pi$ into equivalence 
classes according to the 
sequence of advice on the first $\tau$ rounds.  The algorithm
constructs a finite set of policies $\Pi' $ 
by taking one (arbitrary) policy from each equivalence class, and 
runs \texttt{Exp4.P} for the remaining $T-\tau$ steps using
$\Pi'$ as its set of experts.

For a set of policies $\Pi$, define $G_{\max(\Pi)}$ as the return of the best
policy in $\Pi$ at time horizon $T$.

\begin{thm}\label{thm:VC}  For all distributions $D$ over contexts and 
rewards, for all sets of policies $\Pi$ with VC dimension $d$, 
with probability $1-\delta$,
\[ G_{\textrm{VE}} \geq G_{\max(\Pi)} - 9\sqrt{2T \left(d \ln \frac{eT}{d} +\ln \frac{2}{\delta} \right) }. \]
%when contexts and rewards are drawn i.i.d. according to $D$.
\end{thm}

\begin{proof}
The regret of the initial exploration is bounded by $\tau$.
We first bound the regret of \texttt{Exp4.P} to $\Pi'$, and the regret
of $\Pi'$ to $\Pi$.  We then optimize with respect to $\tau$ to
get the result.

Sauer's lemma implies that
$|\Pi'| \leq \left( \frac{e \tau}{d} \right)^d$ and hence with
probability $1-\delta/2$, we can bound $G_{\ExpP}(\Pi',T-\tau)$
from below by 
\begin{eqnarray*}
G_{\max(\Pi')} - 
6\sqrt{2(T-\tau)(d \ln(e \tau/ d)+ \ln(2/ \delta))}.
\end{eqnarray*}

To bound the regret of $\Pi'$ to $\Pi$, 
pick any sequence of feature observations $x_1,...,x_T$.  
Sauer's Lemma implies the number of
unique functions on the observation sequence in $\Pi$ is bounded by
$\left( \frac{e T}{d} \right)^d$.

For a uniformly random subset $S$ of size $\tau$ of the feature
observations we bound the probability that two functions $\pi,\pi'$
agree on the subset.  Let $n=n(\pi,\pi')$ be the number of disagreements on the
$T$-length sequence.
Then
\[ \mathbf{Pr}_{S}\left[ \forall x \in S\,\,\, \pi(x)=\pi'(x) \right] = \left(1 - \frac{n}{T} \right)^\tau \leq e^{-\frac{n\tau}{T}}.
\]
Thus for all $\pi,\pi'\in \Pi$ with $n(\pi,\pi') \geq \frac{T}{\tau} \ln 1/\delta_0$, we have
$\mathbf{Pr}_{S}\left[ \forall x \in S\,\,\, \pi(x)=\pi'(x) \right] \leq \delta_0$.

Setting $\delta_0 = \frac{\delta}{2} \left( \frac{d}{e T} \right)^{2d}$ and
using a union bound over every pair of policies, we get
\begin{eqnarray*}
\mathbf{Pr}_{S}(
\exists \pi,\pi' & \textrm{ s.t.\ } n(\pi,\pi') \geq  \frac{T}{\tau} \left(2d \ln \frac{eT}{d} +\ln \frac{2}{\delta}\right) \\
& \mbox{s.t.}\ \forall x \in S\,\,\, \pi(x)=\pi'(x) ) \leq \delta/2.
 \end{eqnarray*}
In other words, for all sequences $x_1,...,x_T$ with probability $1 - \delta/2$ over a random subset of size $\tau$
\[ G_{\max(\Pi')} \geq G_{\max(\Pi)} - \frac{T}{\tau} \left(2d \ln \frac{eT}{d} +\ln \frac{2}{\delta}\right).\]

Because the above holds for any sequence $x_1,...,x_T$, it holds in
expectation over sequences drawn i.i.d.\ from $D$.  Furthermore, we can
regard the first $\tau$ samples as the random draw of the subset since
i.i.d.\ distributions are exchangeable.

Consequently, with probability $1-\delta$, we have
\begin{eqnarray*} G_{\textrm{VE}} &\ge& G_{\max(\Pi)} - \tau - \frac{T}{\tau} \left(2d \ln \frac{eT}{d} +\ln \frac{2}{\delta}\right) \\
&& - 6\sqrt{2T(d \ln(e \tau/ d)+ \ln(2/ \delta))}.
\end{eqnarray*}
Letting $\tau = \sqrt{T (2d \ln \frac{eT}{d} +\ln \frac{2}{\delta})}$ and substituting $T \geq \tau$ we get
\[ G_{\textrm{VE}} \geq G_{\max(\Pi)} - 9\sqrt{2T (d \ln \frac{eT}{d} +\ln \frac{2}{\delta})}. \]
\end{proof}

%The proof is in the Appendix.  The same technique easily extends 
This theorem easily extends
to more than two actions ($K > 2$) given
generalizations of the VC-dimension to multiclass classification 
%and generalized versions
and of Sauer's lemma~\cite{HausslerL95}.

\section{A PRACTICAL IMPROVEMENT TO EXP4.P}\label{s:msimp}
Here we give a variant of Step $2$ of Algorithm~\ref{a:exp4.p}
for setting the probabilities $p_j(t)$,
in the style of \emcite{McMahanS09}.  For our analysis of \texttt{Exp4.P},
the two properties we need to ensure in setting the probabilities $p_j(t)$ are
\begin{enumerate}
\item $p_j(t) \approx \sum_{i=1}^{N}{\frac{w_i(t)\xi_j^i(t)}{W_t}}$.
\item The value of each $p_j(t)$ is at least $\m$.
\end{enumerate}
One way to achieve this, as is done in Algorithm~\ref{a:exp4.p}, is to
mix in the uniform distribution over all arms.  While this yields a
simpler algorithm and achieves optimal regret up to a multiplicative
constant, in general, this technique can add unnecessary probability
mass to badly-performing arms; for example it can double the
probability of arms whose probability would already be set to $\m$.

\begin{algorithm}[t]
\caption{An Alternate Method for Setting Probabilities in Step 2 of Algorithm~\protect\ref{a:exp4.p}}
\label{a:probset} 
\textbf{parameters:} $w_1(t), w_2(t), \ldots w_N(t)$ and $\bxi^{1}(t),\ldots,\bxi^{N}(t)$ and 
$\m$\\

set \[ W_{t}=\sum_{i=1}^{N}w_{i}(t)\]

for $j = 1$ to $K$ set \[p_{j}(t)=\sum_{i=1}^{N}\frac{w_{i}(t)\xi_{j}^{i}(t)}{W_{t}}\]

let $\Delta := 0$ and $l := 1$ \\

\textbf{for each} action $j$ in increasing order according to $p_j$
\begin{enumerate}
\item \textbf{if} $p_j \left(1- \Delta/l \right) \geq \m$ \\
  for all actions $j'$ with $p_{j'}\geq p_j$\\
  $p'_{j'} = p_{j'} \left(1- \Delta/l \right)$\\
  return $\forall j\,\,\,p'_j$
\item \textbf{else} 
 $p'_j = \m$, 
 $\ \Delta := \Delta + p'_j - p_j$,
 $\ l := l - p_j$.
\end{enumerate}

\end{algorithm}

A fix to this, first suggested by \emcite{McMahanS09}, is to ensure
the two requirements via a different mechanism.  We present a variant
of their suggestion in Algorithm~\ref{a:probset}, which can be used to
make \texttt{Exp4.P} perform better in practice with a computational
complexity of $O(K \ln K)$ for computing the probabilities $p_j(t)$ per round.
The basic intuition of this algorithm is
that it enforces the minimum probability in order from smallest to
largest action probability, while otherwise minimizing the ratio of
the initial to final action probability.

This technique ensures our needed properties, and it is easy to verify
that by setting probabilities using Algorithm~\ref{a:probset} the
proof in Section~\ref{ss:main} remains valid with little modification.
We use this variant in the experiments in Section~\ref{s:experiments}.

\section{EXPERIMENTS}\label{s:experiments}

In this section, we applied \texttt{Exp4.P} with the improvement in Section~\ref{s:msimp} to a large-scale contextual bandit problem.  The purpose of the experiments is two-fold: it gives a proof-of-concept demonstration of the performance of \texttt{Exp4.P} in a non-trivial problem, and also illustrates how the algorithm may be implemented efficiently for special classes of experts.

The problem we study is personalized news article recommendation on the Yahoo!\ front page~\cite{AgarwalCEMPRRZ08,LCLS10}.  Each time a user visits the front page, a news article out of a small pool of hand-picked candidates is highlighted.  The goal is to highlight the most interesting articles to users, or formally, maximize the total number of user clicks on the recommended articles.  In this problem, we treat articles as arms, and define the payoff to be $1$ if the article is clicked on and $0$ otherwise.  Therefore, the average per-trial payoff of an 
algorithm/policy is the overall \textbf{click-through rate} (or \textbf{CTR} for short).

Following \namecite{LCLS10}, we created $B=5$ user clusters and thus each user, based on \emph{normalized} Euclidean distance to the cluster centers, was associated with a $B$-dimensional \emph{membership feature} $\mathbf{d}$ whose (non-negative) components always sum up to 1.
Experts are designed as follows.  Each expert
is associated with a mapping from user clusters to articles, that is, with a
vector $\mathbf{a}\in\{1,\ldots,K\}^B$ where $a_b$ is the article to be displayed
for users from cluster $b\in\{1,\ldots,B\}$.  When a user arrives with
feature $\mathbf{d}$, the prediction $\bxi^{\mathbf{a}}$ of expert $\mathbf{a}$ is
$\xi^{\mathbf{a}}_j=\sum_{b:a_b=j}d_b$.  There are a total of $K^B$ experts.

Now we show how to implement \texttt{Exp4.P} efficiently.
Referring to the notation in \texttt{Exp4.P}, we have
\begin{eqnarray*}
  \hat{y}_{\veca}(t) &=& \preda(t) \cdot \hat{\rr}(t) 
    = \sum_j \sum_{b:a_b=j} d_b (t) \hat{r}_{j}(t) \\
   && = \sum_b d_b(t) \hat{r}_{a_b}(t),
\\
  \hat{v}_{\veca}(t) &=& \sum_j \sum_{b:a_b=j} \frac{d_b(t)}{p_j(t)} 
  = \sum_b \frac{d_b(t)}{p_{a_b}(t)}.
\end{eqnarray*}
Thus,
\begin{align*}
w_{\veca}&(t+1) \\
&= w_{\veca}(t) \exp\left(\frac{\m}{2}\left(\hat{y}_\veca(t)+\hat{v}_\veca(t) \sqrt{\frac{\ln(N/\delta)}{KT}}\right) \right) \\
&= w_{\veca}(t) \exp\left(\sum_b d_b(t) f_{a_b}(t)\right),
\end{align*}
where 
\[
f_j(t) = \frac{\m}{2}\left(\hat{r}_j(t) + \frac{1}{p_j(t)}\sqrt{\frac{\ln(N/\delta)}{KT}}\right).
\]
Unraveling the recurrence, we  rewrite $w_{\veca}(t+1)$ by
\begin{eqnarray*}
w_{\veca}(t+1) &=& \exp\left(\sum_{\tau=1}^t \sum_b d_b(\tau) f_{a_b}(\tau)\right) \\
&=& \exp\left(\sum_b \sum_{\tau=1}^t d_b(\tau) f_{a_b}(\tau)\right) \\
&=& \prod_b g_{b,a_b}(t),
\end{eqnarray*}
implying that $w_{\veca}(t+1)$ can be computed implicitly by maintaining the quantity
$g_{b,j}(t) = \exp\left(\sum_{\tau=1}^t d_b(\tau) f_j(\tau)\right)$ for each $b$ and $j$.
Next, we compute $W_t$ as follows:
$W_t = \sum_{\veca} w_\veca(t)
  = \sum_{\veca} \prod_b g_{b,a_b}(t)
  = \prod_b \left( \sum_j g_{b,j}(t) \right).$
Repeating the same trick, we have
\[
\sum_{\veca}\frac{w_{\veca}(t)\preda(t)}{W_t} = \sum_b \frac{d_b(t) g_{b,j}(t)}{\sum_{j'=1}^K g_{b,j'}(t)},
\]
which are the inputs to Algorithm~\ref{a:probset} to produce the final arm-selection
probabilities, $p_j(t)$ for all $j$.
Therefore, for this structured set of experts, the time complexity of
\texttt{Exp4.P} is only linear in $K$ and $B$
despite the exponentially large size of this set.

To compare algorithms, we collected historical user visit events with
a random policy that chose articles uniformly at random for a fraction of
user visits on the Yahoo!\ front page
from May 1 to 9, 2009.  This data contains over $41$M user visits, a total of
$253$ articles, and about $21$ candidate articles in the pool per user visit.
(The pool of candidate articles changes over time, requiring corresponding
modifications to \texttt{Exp4.P}\footnote{Our modification ensured that a new article's initial score was the average of all currently available ones'.}).  With such random traffic data,
we were able to obtain an unbiased \emph{estimate of the CTR} (called \textbf{eCTR}) 
of a bandit algorithm as if it is run in the real world~\cite{LCLS10}.

Due to practical concerns when applying a bandit algorithm,
%as in our news recommendation problem, 
it is common to randomly assign each
user visit to one of two ``buckets'': the \emph{learning bucket}, where the
bandit algorithm is run, and the \emph{deployment bucket}, where the
greedy policy (learned by the algorithm in the learning bucket)
is used to serve users without receiving payoff information.  Note that since
the bandit algorithm continues to refine its policy based on payoff feedback
in the learning bucket, its greedy policy may change over time.  Its eCTR in the
deployment bucket thus measures how good this
%non-stationary 
greedy policy is.
And as the deployment bucket is usually much larger than the learning
bucket, the deployment eCTR is deemed a more important metric.
Finally, to protect business-sensitive information, we only report
\emph{normalized eCTR}s, which are 
%defined as 
the actual eCTRs divided
by the random policy's eCTR.

Based on estimates of $T$ and $K$, we ran \texttt{Exp4.P} with $\delta=0.01$.
The same estimates were used to set $\gamma$ in
\texttt{Exp4} to minimize the regret bound in Theorem~7.1 of \namecite{AuerCFS02}.
Table~\ref{tbl:results} summarizes eCTRs of all three algorithms in the
two buckets.  All differences are significant due to the large
volume of data.

\begin{table}
\begin{center}
\begin{tabular}{l|c|c|c}
& \texttt{Exp4.P} & \texttt{Exp4} & $\epsilon$-greedy \\
 \hline
learning CTR & $1.0525$ & $1.0988$ & $1.3827$ \\
deployment CTR & $1.6512$ & $1.5309$ & $1.4290$
\end{tabular}
\end{center}
\caption{Overall click-through rates (eCTRs) of various algorithms on the May 1--9 data set.} \label{tbl:results}
\end{table}

First, \texttt{Exp4.P}'s eCTR is slightly worse than \texttt{Exp4} in
the learning bucket.  This gap is probably due to the more conservative nature of \texttt{Exp4.P}, as it uses the additional $\hat{v}_i$ terms to control variance,
which in turn encourages further exploration.  In return for the more extensive
exploration, \texttt{Exp4.P} gained the highest deployment eCTR, implying its greedy
policy is superior to \texttt{Exp4}. 

Second, we note a similar comparison to the $\epsilon$-greedy variant of
\texttt{Exp4.P}.  It was the most greedy among the three algorithms and thus had
the highest eCTR in the learning bucket, but lowest eCTR in the deployment bucket.
This fact also suggests the benefits of using the somewhat more complicated
soft-max exploration scheme in \texttt{Exp4.P}.

\subsubsection*{Acknowledgments}

We thank Wei Chu for assistance with the experiments and Kishore Papineni for helpful discussions.

This work was done while Lev Reyzin and Robert E.\ Schapire were at Yahoo!\ Research, NY.
Lev Reyzin acknowledges this material is based upon work supported by the NSF under
Grant \#0937060 to the CRA for the Computing Innovation Fellowship program.

\bibliography{paper_arxiv}

\appendix

\section{PROOF OF LEMMA 4}

Recall that the estimated reward of expert $i$ is defined as \[
\hat{G}_{i}\doteq\sum_{t=1}^{T}\hat{y}_{i}(t).\]
Also 
\[
\hat{\sigma}_{i}\doteq\sqrt{KT}+\frac{1}{\sqrt{KT}}\sum_{t=1}^{T} \vh_i(t)
\]
and that
\[
   \ucb = \max_i \paren{\hat{G}_i + \hat{\sigma}_i
                             \cdot\sqrt{\ln(N/\delta)} }.
\]

{\bf Lemma 4.}\quad
Under the conditions of Theorem 2,
\begin{eqnarray*}
G_{\ExpP} &\ge&
   \left(1-2\sqrt{\frac{K \ln N}{T}}\right)\ucb
   - 2\sqrt{KT \ln(N/\delta)}\\
& &
   - \sqrt{KT \ln N}
   - \ln(N/\delta).
   \end{eqnarray*}

\begin{proof}
For the proof, we use $\gamma = \sqrt{\frac{K \ln N}{T}}$.\\
We have $$p_{j}(t)\ge \m = \sqrt{\frac{\ln N}{KT}}$$ and $$\hat{r}_{j}(t)\le 1/\m$$
so that 
\[
\hat{y}_i(t) \leq 1/\m \ \ \ \ \  \mathrm{and} \ \ \ \ \  \vh_i(t)\leq 1/\m.
\]
Thus, %\begin{eqnarray*}
\begin{eqnarray*}
\frac{\m}{2}\paren{\hat{y}_{i}(t)+\sqrt{\frac{\ln(N/\delta)}{KT}} \vh_i(t)}
&\leq&
\frac{\m}{2} (\hat{y}_{i}(t)+\vh_i(t)) \\
&\leq& 1.
\end{eqnarray*}
%\end{eqnarray*}

Let $\bar{w}_{i}(t)=w_{i}(t)/W_{t}$. We will need the following inequality:
\begin{inequality}\label{ineq1}\quad
$\sum_{i}^{N} \bar{w}_{i}(t) \vh_i(t) \le \frac{K}{1-\gamma }$.
\end{inequality}
As a corollary, we have
\begin{eqnarray*}
\sum_{i}^{N} {\bar{w}_{i}(t)}{\vh_{i}(t)^2} & \le & 
\sum_{i}^{N} {\bar{w}_{i}(t)}{\vh_{i}(t)}\frac{1}{\m}\\
& \le & \sqrt{\frac{KT}{\ln N}} \frac{K}{1-\gamma}.
\end{eqnarray*}

Also, \emcite{AuerCFS02} (on p.67) prove the following two inequalities (with a typo).  For completeness, the proofs of all three inequalities are given below 
this proof.

\begin{inequality}\label{ineq3}\quad
$ \sum_{i=1}^{N}\bar{w}_{i}(t)\hat{y}_{i}(t)\le\frac{r_{j_t}(t)}{1-\gamma}$.
\end{inequality}

\begin{inequality}\label{ineq4}\quad
$
\sum_{i=1}^{N}\bar{w}_{i}(t)\hat{y}_{i}(t)^{2}\le\frac{\hat{r}_{j_t}(t)}{1-\gamma}.$
\end{inequality}

Now letting $b=\frac{\m}{2}$
and
$c=\frac{\m \sqrt{\ln (N/\delta)}}{2\sqrt{KT}}$
we have
\begin{eqnarray}
\frac{W_{t+1}}{W_{t}} & = & \sum_{i=1}^{N}\frac{w_{i}(t+1)}{W_{t}}\nonumber \\
& = & \sum_{i=1}^{N}\bar{w}_{i}(t)\exp \left(b\hat{y}_i(t)+
                      c\vh_i(t)\right)\nonumber
 \\
\label{eqn:2ineq} & \le & \sum_{i=1}^{N}\bar{w}_{i}(t)\left[1
               +b\hat{y}_{i}(t)
               +c\vh_i(t) \right] \\
               & &+
                  \sum_{i=1}^{N}\bar{w}_{i}(t)\left[ 2b^2\hat{y}_{i}(t)^2
               +2c^2\vh_i(t)^2
          \right]\nonumber
\\
&=&
     1
               +b \sumi\bar{w}_{i}(t)\hat{y}_{i}(t)
               +c \sumi\bar{w}_{i}(t) \vh_i(t)\nonumber \\
               &&
               +2b^2\sumi\bar{w}_{i}(t)\hat{y}_{i}(t)^2
               +2c^2\sumi\bar{w}_{i}(t) \vh_i(t)^2\nonumber
\\
&\leq&
\label{eqn:4ineq}  1
               +b\frac{r_{j_t}(t)}{1-\gamma}
               +c\frac{{K}}{1-\gamma} 
               +2b^2\frac{\hat{r}_{j_t}(t)}{1-\gamma} \\
               & &
               +2c^2\sqrt{\frac{KT}{\ln N}}\frac{K}{1-\gamma}\nonumber.
\end{eqnarray}
Eq.\ (\ref{eqn:2ineq}) uses $e^{a}\le1+a+ (e-2)a^{2}$ for $a\le1$, $(a+b)^2 \leq 2a^2 + 2 b^2$, and $e-2 < 1$.
Eq.\ (\ref{eqn:4ineq}) uses inequalities 1 through 3.

Now take logarithms, use the inequality $\ln(1+x)\le x$, sum both sides
over $T$, and we obtain
\begin{eqnarray*}
\ln\left(\frac{W_{T+1}}{W_{1}}\right)
 &\le&
               \frac{b}{1-\gamma}\sumt {r_{j_t}(t)}
               +c\frac{KT}{1-\gamma} \\
               &&+\frac{2b^2}{1-\gamma} \sumt \hat{r}_{j_t}(t)
               +2c^2\sqrt{\frac{KT}{\ln N}}\frac{KT}{1-\gamma}
\\
 &\le&
         \frac{b}{1-\gamma} G_{\textrm{Exp4.P}}
               +c\frac{KT}{1-\gamma} 
               +\frac{2b^2}{1-\gamma} K \ucb \\
& &               +2c^2\sqrt{\frac{KT}{\ln N}}\frac{KT}{1-\gamma}.
\end{eqnarray*}
Here, we used 
\[  G_{\textrm{Exp4.P}} = \sumt {r_{j_t}(t)} \]
and
\[
 \sum_{t=1}^{T}\hat{r}_{j_t}(t)=K\sum_{t=1}^{T}\frac{1}{K}\sum_{j=1}^{K}\hat{r}_{j}(t)\le
 K\hat{G}_{\textrm{uniform}}\le K \ucb.
\]
because we assumed that the set of experts includes one who always selects
each action uniformly at random.

We also have $\ln(W_{1})=\ln(N)$ and
\begin{eqnarray*}
\ln(W_{T+1}) &\ge& \max_i \left( \ln w_{i}(T+1) \right) \\
  & = &
\max_i \paren{b\hat{G}_i + c \sum_{t=1}^{T} \vh_i(t)}\\
&=&
b \ucb - b \sqrt{KT \ln (N/\delta)}.
\end{eqnarray*}
Combining then gives
\begin{eqnarray*}   
&b \ucb - b \sqrt{KT \ln (N/\delta)} - \ln N &  \\
& \leq &\\
&         \frac{b}{1-\gamma} G_{\textrm{Exp4.P}}
               +c\frac{KT}{1-\gamma} 
         +\frac{2b^2}{1-\gamma} K \ucb 
        +2c^2\sqrt{\frac{KT}{\ln N}}\frac{KT}{1-\gamma}.&
\end{eqnarray*}
Solving for $G_{\textrm{Exp4.P}}$ now gives
\begin{eqnarray}
\nonumber
G_{\textrm{Exp4.P}}
&\geq&
   \left(1-\gamma-2bK\right)\ucb
   - \left(\frac{1-\gamma}{b}\right)\ln N\nonumber \\
& &   - (1-\gamma) \sqrt{KT \ln(N/\delta)} 
   - \frac{c}{b}KT\nonumber  \\
& &   -2\frac{c^2}{b}\sqrt{\frac{KT}{\ln N}}KT\nonumber
   \nonumber
\\
\label{eq:5}
&\geq&
   \left(1-\gamma-2bK\right)\ucb
   - \sqrt{KT \ln(N/\delta)} \\
& &   - \frac{1}{b}\ln N 
   - \frac{c}{b}KT 
   -2\frac{c^2}{b}\sqrt{\frac{KT}{\ln N}}KT\nonumber\\
&=&
\label{eq:6}
   \left(1-2\sqrt{\frac{K \ln N}{T}}\right)\ucb
      - \ln(N/\delta)\\
& &   - 2 \sqrt{KT \ln N} 
      - \sqrt{KT \ln(N/\delta)},\nonumber
\end{eqnarray}
using $\gamma>0$ in \eqref{eq:5} and plugging in the definition of $\gamma,b,c$ in \eqref{eq:6}.
\end{proof}

We prove Inequalities \ref{ineq1} through \ref{ineq4} below.

Let $\bar{w}_{i}(t)=w_{i}(t)/W_{t}$. 
\begin{inequalityrestate}\quad
$\sum_{i}^{N} \bar{w}_{i}(t) \vh_i(t) \le \frac{K}{1-\gamma }$.
\end{inequalityrestate}
\begin{proof}
\begin{eqnarray*}
\sum_{i}^{N}\bar{w}_{i}(t)\vh_i(t) & = &
\sum_{i}^{N}\bar{w}_{i}(t)\sum_{j}^{K}\frac{\xi_{j}^{i}(t)}{p_{j}(t)}\\
 & = & \sum_{j=1}^{K}\frac{1}{p_{j}(t)}\sum_{i}^{N}\bar{w}_{i}(t)\xi_{j}^{i}(t)\\
 & = & \sum_{j=1}^{K}\frac{1}{p_{j}(t)}\left(\frac{p_{j}(t)- \m }{1-\gamma}\right)\\
 & \le & \sum_{j=1}^{K}\frac{1}{1-\gamma}\\
 & = & \frac{{K}}{1-\gamma}.
 \end{eqnarray*} 
\end{proof}
 
\begin{inequalityrestate}\quad
$ \sum_{i=1}^{N}\bar{w}_{i}(t)\hat{y}_{i}(t)\le\frac{r_{j_t}(t)}{1-\gamma}$.
\end{inequalityrestate}
\begin{proof}
\begin{eqnarray*}
   \sum_{i=1}^{N}\bar{w}_{i}(t)\hat{y}_{i}(t)&=&
   \sum_{i=1}^{N}\bar{w}_{i}(t)\left(\sum_{j=1}^{K}{\xi_j^i(t)\hat{r}_j(t)}\right)\\
   &= & \sum_{j=1}^{K}\left( \sum_{i=1}^{N}{\bar{w}_i(t)\xi^{i}_j(t)} \right)\hat{r}_j(t)\\
   &=& \sum_{j=1}^{K}\left(\frac{p_{j}(t)- \m }{1-\gamma}\right)\hat{r}_j(t)\\
   &\le&\frac{r_{j_t}(t)}{1-\gamma}.
\end{eqnarray*}
\end{proof}

\begin{inequalityrestate}\quad
$
\sum_{i=1}^{N}\bar{w}_{i}(t)\hat{y}_{i}(t)^{2}\le\frac{\hat{r}_{j_t}(t)}{1-\gamma}.$
\end{inequalityrestate}
\begin{proof}
\begin{eqnarray*}
\sum_{i=1}^{N}\bar{w}_{i}(t)\hat{y}_{i}(t)^{2}&=&
 \sum_{i=1}^{N}\bar{w}_{i}(t)\left(\sum_{j=1}^{K}{\xi_j^i(t)\hat{r}_j(t)}\right)^2\\
&=&\sum_{i=1}^{N}\bar{w}_{i}(t)\left(\xi_{j_t}^i(t)\hat{r}_{j_t}(t)\right)^2\\
&\le&\left(\frac{p_{j_t}(t)}{1-\gamma}\right)\hat{r}_{j_t}(t)^2\\
&\le&\frac{\hat{r}_{j_t}(t)}{1-\gamma}.
\end{eqnarray*}
\end{proof}

% \vfill

%\pagebreak

\end{document}